\documentclass[sn-nature]{sn-jnl}

\usepackage{graphicx}%
\usepackage{multirow}%
\usepackage{amsmath,amssymb,amsfonts}%
\usepackage{amsthm}%
\usepackage{mathrsfs}%
\usepackage[title]{appendix}%
\usepackage{xcolor}%
\usepackage{textcomp}%
\usepackage{manyfoot}%
\usepackage{booktabs}%
\usepackage{algorithm}%
\usepackage{algorithmicx}%
\usepackage{algpseudocode}%
\usepackage{listings}

\usepackage[utf8]{inputenc} 
\usepackage[T1]{fontenc}    
\usepackage{hyperref}       
\usepackage{url}            
\usepackage{nicefrac}       
\usepackage{microtype}      
\usepackage{bm}
\usepackage{array}
\usepackage{svg}
\usepackage{wrapfig}
\usepackage{subcaption}
\usepackage{rotating}
\usepackage{adjustbox}
\usepackage{multicol}

\definecolor{dark-yellow}{HTML}{eeae00}
\definecolor{dark-green}{HTML}{47e11e}
 
\def\vh{{\bm{h}}}
\def\vx{{\bm{x}}}
\def\mA{{\bm{A}}}
\def\mB{{\bm{B}}}
\def\mC{{\bm{C}}}
\def\mP{{\bm{P}}}
\def\mQ{{\bm{Q}}}
\def\mI{{\bm{I}}}
\def\sR{{\mathbb{R}}}
\def\sN{{\mathbb{N}}}

\newtheorem{proposition}{Proposition}
\newtheorem{problem}{Problem}

\begin{document}

\title[SeRpEnt: Selective Resampling for Expressive State Space Models]{SeRpEnt: Selective Resampling for Expressive State Space Models}


\author[1]{\fnm{Stefano} \sur{Rando}}\email{stefanorando@italailabs.com}

\author[2]{\fnm{Luca} \sur{Romani}}\email{romani.2005510@studenti.uniroma1.it}

\author[2]{\fnm{Matteo} \sur{Migliarini}}\email{matteo.migliarini@uniroma1.it}
\author[1]{\fnm{Luca} \sur{Franco}}\email{lucafranco@italailabs.com}

\author[3]{\fnm{Denis} \sur{Gudovskiy}}\email{denis.gudovskiy@us.panasonic.com}

\author[2]{\fnm{Fabio} \sur{Galasso}}\email{galasso@di.uniroma1.it}

\affil[1]{\orgname{ItalAI S.R.L}}
\affil[2]{\orgname{Sapienza University of Rome}}
\affil[3]{\orgname{Panasonic AI Lab}}




\abstract{State Space Models (SSMs) have recently enjoyed a rise to prominence in the field of deep learning for sequence modeling, especially as an alternative to Transformers. Their success stems from avoiding two well-known drawbacks of attention-based models: quadratic complexity with respect to the sequence length and inability to model long-range dependencies. The SSM variant Mamba has demonstrated performance comparable to Transformers without any form of attention, thanks to the use of a selective mechanism for the state parameters. Selectivity, however, is only evaluated empirically and the reasons of its effectiveness remain unclear. In this work, we show how selectivity is related to the sequence processing. Our analysis shows that selective time intervals in Mamba act as linear approximators of information. Then, we propose our SeRpEnt architecture, a SSM that further exploits selectivity to compress sequences in an information-aware fashion. It employs a resampling mechanism that aggregates elements based on their information content. Our empirical results in the Long Range Arena benchmark and other language modeling tasks show benefits of the SeRpEnt's resampling mechanism.}




\maketitle

\section{Introduction}
\label{sec:intro}

Sequence models, and more specifically language models, have gained the spotlight in the deep learning research due to their role as the backbone of foundational models (FMs) \citep{Achiam2023GPT4TR, Touvron2023LLaMAOA, Touvron2023Llama2O}. Overwhelmingly, their architecture is based on the Transformer \citep{NIPS20173f5ee243} with its attention mechanism \citep{bahdanau2014neural}. This approach is justified by the impressive body of empirical results in numerous applications \citep{ISLAM2024122666}. The Transformer, however, does not come without its own drawbacks. The attention layer enables the dense processing of elements in a sequence at the cost of a computational burden that scales quadratically with the context length. Besides the computational limitations, Transformers and sparse attention models \citep{child2019generating} are inefficient at modeling long-range dependencies, both empirically \citep{liu2023lost} and in \textit{ad-hoc} benchmarks \citep{tay2020long}.

State Space Models (SSMs) \citep{10.1115/1.3662552} are a class of models thoroughly studied in many established scientific areas. Recently, the HiPPO \citep{gu2020hippo} has paved the way in deep learning for sequence modeling by showing how SSMs can be leveraged for long-range dependencies with a careful initialization of the state parameters. Then, a class of SSM variants has been built on top of HiPPO, starting from the structured SSM (S4) \citep{gu2022efficiently} to, eventually, Mamba \citep{gu2023mamba}. The latter is the first SSM to match Transformers' performance without resorting to any form of attention. The main Mamba novelty is the selectivity, where state space parameters are themselves outputs of time-dependent neural networks \citep{45823hypernetorks}. While the idea is supported by compelling experiments, however, the source of its added expressivity is left unexplored with a room for further analytical and architectural exploration.

\begin{figure}[t]
    \centering
    \includegraphics[width=0.95\linewidth]{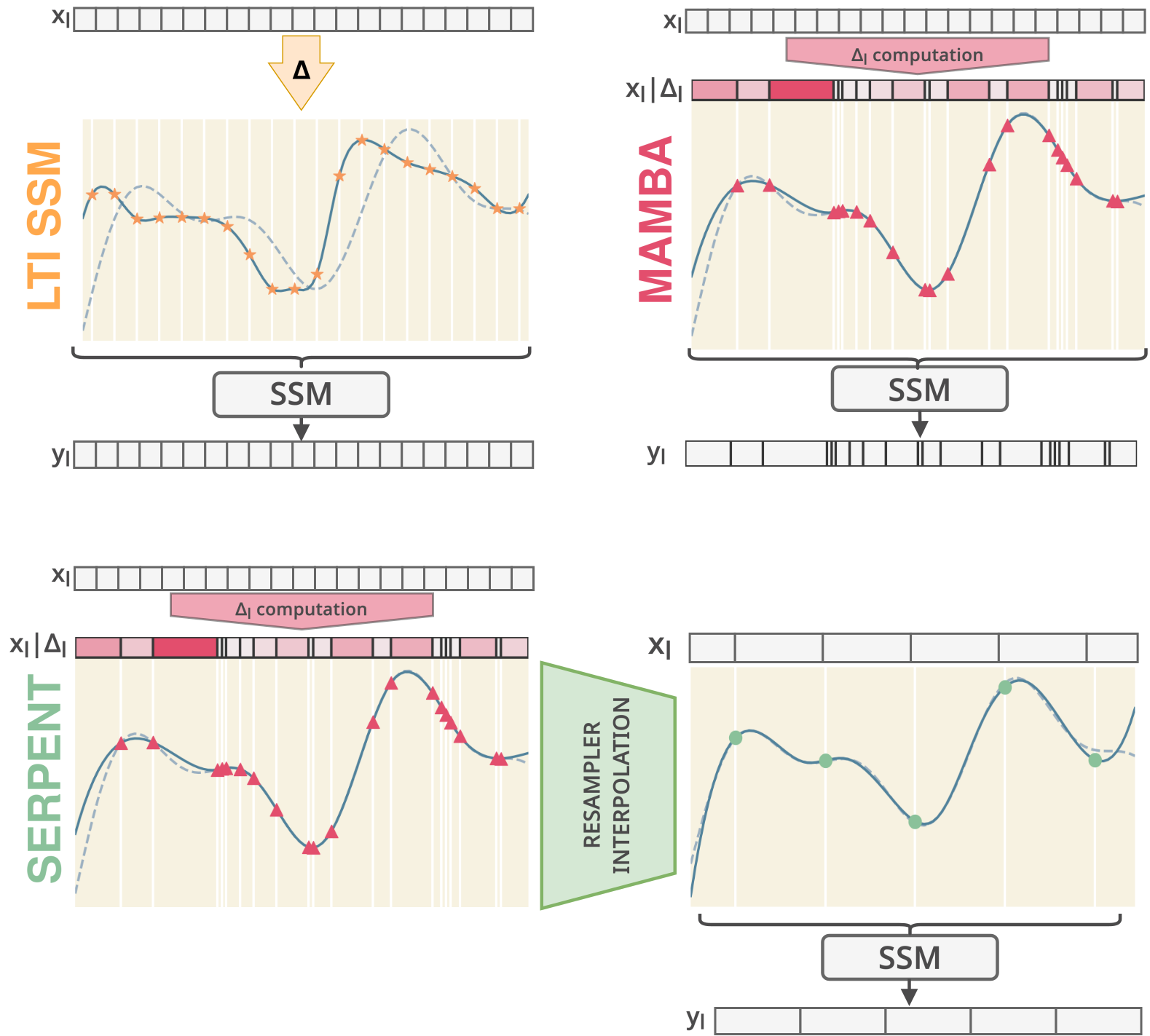}
    \caption{(\textbf{Top Left}) A linear time-invariant SSM with uniform sampling assumes a sequence $\{ x_l \}$ is sampled from an underlying continuous function which incurs a loss of precision due to the constant time interval $\Delta$. (\textbf{Top Right}) The Mamba computes time-variable intervals $\Delta_l$ that are dependent on the sequence elements: the sequence is assumed to be non-uniformly sampled from the underlying function. (\textbf{Bottom}) Our SeRpEnt assumes time-dependent intervals $\Delta_l$ and, additionally, resamples the sequence through interpolation by aggregating together elements whose information is related.}
    \label{fig:resampling}
\end{figure}



Stemming from a renewed analysis of the selectivity mechanism in Mamba, we propose \textbf{SeRpEnt}, a \textit{selective resampling procedure for sequence compression}. We show that selectivity learns to discriminate sequence elements based on amount of information in them. Using this analysis, SeRpEnt compresses sequences in an information-aware fashion: an overall reduction of the sequence length improves the global processing, while local elements are aggregated based on their information content. Figure \ref{fig:resampling} visualizes the main SeRpEnt novelty when compared to previous works.




Our main contributions are summarized as follows:
\begin{itemize}
    \itemsep0em
    \item We show that selectivity is rooted in information processing capabilities, in particular the selective time intervals learned by Mamba act as linear approximators of information.
    \item We develop a selective resampling strategy for SSMs and propose our \textbf{SeRpEnt}, a model that compresses sequences in an information-aware fashion and, hence, enables more efficient global processing by aggregating together elements based on their information content.
    \item We showcase SeRpEnt's performance in the Long Range Arena and other language modeling tasks to advance the development of SSMs as an alternative to Transformers.

\end{itemize}

\section{Related Work}
\label{sec:related}


Despite progress in the field of machine learning for sequence processing \citep{Lim2020TimeseriesFW}, capturing long-range dependencies has proven a difficult problem to tackle \citep{CHEN2023101819}.  Particularly, this is evident in scenarios characterized by large contextual windows in the domain of Natural Language Processing (NLP) \citep{manning1999foundations}, where the retrieval of specific information in large textual corpora presents a challenging task \citep{khandelwal2018sharp, kandpal2023large}. In this section, we review recent mainstream approaches for the problem of modeling long-range dependencies. We discuss their properties and rationale, as well as their shortcomings.

\subsection{Attention Mechanism in Transformers}

The attention mechanism \cite{bahdanau2014neural, NIPS20173f5ee243} has become ubiquitous in sequence modeling, thanks to its capability of relating tokens in a sequence. That justifies approaches for long-sequence modeling built upon its cornerstone architecture: the Transformer. However, the Transformer presents a dual problem in the specific context of long-range dependencies: the computational complexity of inferring a token is quadratic in the length of the sequence, and its expressivity is negatively affected when applied to long sequences \citep{liu2023lost, li2024longcontext}. 

To overcome the first shortcoming, alternative forms of attention have been proposed \citep{wang2020linformer, kitaev2020reformer, child2019generating}, which rely on sparse representations and computations. More recently, \cite{dao2022flashattention} develop a fast implementation which exploits the GPU hardware to minimize data exchange. While sparse attention successfully deals with the computational side of the problem, their second drawback, \emph{i.e.} the expressivity for modeling long-range sequences, remains a major obstacle in many practical applications \citep{zhang2023cab}.

\subsection{State Space Models}

State space models are an established family of transformations in traditional statistics. The authors of HiPPO \citep{gu2020hippo} propose the use of SSMs as an alternative backbone for long sequence modeling by leveraging a theory for parameters initialization. Unlike transformers, this approach unlocks inference with linear complexity. \cite{gu2022efficiently} further suggests an efficient convolutional approximation for SSMs which enables concurrent computations at training time. The initial success of SSMs spurred a research in related areas \emph{e.g.} audio generation \citep{pmlr-v162-goel22a}, reinforcement learning \citep{lu2023structured}, and spatiotemporal modeling \citep{smith2023convolutional}.

\cite{fu2023hungry} are the first to explore the use of SSMs for language modeling. More recently, Mamba \citep{gu2023mamba} proposes the first SSM capable of matching Transformer's performance in the NLP tasks. Mamba does so by removing the linear time-invariant constraint in previous works and by using a time-dependent parametric selectivity mechanism. This leads to a fully recurrent model that generalizes prior gated recurrent units \citep{Cho2014LearningPR}. Authors argue that the main source of improvement comes from the model being able to selectively update its memory and decide which information to discard and which to retain. The arguments, however, are only explored empirically and the lack of an appropriate framework for understanding and exploiting the selectivity mechanism spurs future research, including this work.

Mamba spurred an extensive following of works that build upon it \citep{mamba2}, in some cases adapting it to other types of data \citep{liang2024pointmambasimplestatespace, xie2024quadmamba}. The authors of \citep{zhan2024exploringtokenpruningvision} propose to extend Mamba with a token pruning mechanism that resembles our compression strategy. In our case, however, we do not apply direct pruning, but process every token to obtain a shorter sequence to process.

\section{Background}
\label{sec:background}

\subsection{State Space Models}

State space models, a class of sequence transformation models, are defined through a set of first-order differential equations \emph{e.g.}, on real-valued functions $x (t) \in \sR \mapsto y (t) \in \sR$. Specifically, there exists a state vector $\vh \left(  \cdot \right) \in \sR^{N}$ such that the following equations are satisfied:

\begin{minipage}{0.45\textwidth}
\begin{subequations}
\begin{align}
    \dot{\vh}(t) & = \mA(t) \vh(t) + \mB(t) x(t), \label{eq:general_ssm_ode} \\
    y(t) & = \mC(t) \vh(t), \label{eq:general_ssm_output}
\end{align}
\end{subequations}
\end{minipage}
\hfill
\begin{minipage}{0.45\textwidth}
\begin{subequations}
\begin{align}
    \vh_{l} & = \overline{\mA}_l \vh_{l - 1} + \overline{\mB}_l x_l \label{eq:discrete_ssm_ode}, \\
    y_l & = \mC_l \vh_l \label{eq:discrete_ssm_output}
\end{align}
\end{subequations}
\end{minipage}

for matrices $\mA(t) \in \sR^{N \times N}$, $\mB(t) \in \sR^{N \times 1}$ and $\mC(t) \in \sR^{1 \times N}$, which are the state space parameters.

Equations (\ref{eq:discrete_ssm_ode}-\ref{eq:discrete_ssm_output}) express the case of discrete SSMs, where the input $x_l$ and output $y_l$ are $L$-length real-valued sequences rather than functions. In such case, we assume that the sequence $x_l$ is sampled from a continuous function $x \left( \cdot \right)$, where $x_l = x \left( t_l \right)$. Then, the discretized parameters $\overline{\mA}_l$, $\overline{\mB}_l$ are obtained using the zero-order hold \citep{pohlmann2010principles} rule as
\begin{equation}
\label{eq:zoh_disc}
\overline{\mA}_l = \exp({\Delta_l \mA}) ,~\textrm{and}~
\overline{\mB}_l = \left( \Delta_l \mA \right)^{-1} \left( \exp({\Delta_l \mA}) - \mI \right) \Delta_l \mB,
\end{equation}
where $\Delta_l = t_l - t_{l - 1}$ are the time intervals between sampling points.

\paragraph{LTI SSMs as recurrent and convolutional neural networks.} A continuous model is called a linear time-invariant (LTI) SSM when the parameters $\mA, \mB$, and $\mC$ are independent of time $t$. If we also assume constant time intervals $\Delta_l$ for the discrete LTI SSM model \emph{i.e.} the sequence $x_l$ is uniformly sampled from $x$, the equations (\ref{eq:discrete_ssm_ode}-\ref{eq:discrete_ssm_output}) can be simplified to

\begin{minipage}{0.4\textwidth}
\begin{subequations}
\begin{align}
    \vh_l & = \overline{\mA} \vh_{l - 1} + \overline{\mB} x_l, \label{eq:discrete_ltissm_ode} \\
    y_l & = \mC \vh_l, \label{eq:discrete_ltissm_output}
\end{align}
\end{subequations}
\end{minipage}
\hfill
\begin{minipage}{0.5\textwidth}
\begin{subequations}
\begin{align}
    \overline{K} & = (\mC \overline{\mB}, \mC \overline{\mA}^{1} \overline{\mB}, \dots, \mC \overline{\mA}^{l} \overline{\mB} ), \label{eq:kernel_conv} \\
    y & = x \ast \overline{K}. \label{eq:discrete_ssm_conv}
\end{align}
\end{subequations}
\end{minipage}

The model in (\ref{eq:discrete_ltissm_ode}-\ref{eq:discrete_ltissm_output}) can be interpreted as a RNN \citep{Yu2019ARO} with parameters $\overline{\mA}, \overline{\mB}$ and $\mC$. The transformations can also be computed as a convolution (\ref{eq:discrete_ssm_conv}) with the $\overline{K}$ kernel (\ref{eq:kernel_conv}). 

\paragraph{Initialization of parameters.} HiPPO~\citep{gu2020hippo, gu2022train} demonstrates that the initialization of the state space parameters significantly impacts model performance. Their work derives expressions for $\mA$ and $\mB$ that enable an interpretation of the state vectors $\vh \left( \cdot \right)$ as coefficients for an approximation of $x \left( \cdot \right)$ on a polynomial space. Subsequent works \citep{gu2022parameterization, NEURIPS2022_9156b0f6, gu2022parameterization, Yu2023RobustifyingSM} propose $\mA$ simplification to speed up the computation of the kernel $\overline{K}$ in (\ref{eq:kernel_conv}). Particularly, our model relies on the structured SSM (S4) \cite{gu2022efficiently} approach, where the matrix $\mA$ is decomposed into a normal matrix $\mA_{\left( N \right)}$ and low-rank $\mP \mQ^{T}$ components ($\mP, \mQ \in \sR^{N \times 1}$) with the initialization as

\begin{minipage}{0.3\textwidth}
\begin{equation}
    \mA = \mA_{\left( N \right)} + \mP \mQ^{T},
\end{equation}
\end{minipage}
\hfill
\begin{minipage}{0.7\textwidth}
\begin{equation}
\mA^{\mathrm{init}}_{ij} = -
\begin{cases}
    \sqrt{2i + 1} \sqrt{2j + 1} & \text{if} i > j, \\
    i + 1 & \text{if} i = j, \\
    0 & \text{if} i < j.
\end{cases}
\end{equation}
\end{minipage}

\paragraph{Selective SSMs.}

While the initial line of work have been focused on LTI SSMs, Mamba~\citep{gu2023mamba} proposes to use the time-dependent parameters $\mB_l, \mC_l$ and $\Delta_l$. Then, the parameters are functions of the input $x_l$ that can be expressed by
\begin{equation}
\label{eq:mamba_deltan}
    \mB_l = \theta_{\mB} \left( x_l \right),~ 
    \mC_l = \theta_{\mC} \left( x_l \right),~\textrm{and}~
    \Delta_l = \mathrm{softplus} \left( \Delta + \theta_{\Delta} \left( x_l \right) \right),
\end{equation}
where $\theta_{\mB}$, $\theta_{\mC}$, $\theta_{\Delta}$ are linear transformations, and $\Delta_l$ is a learnable parameter. While the matrix $\mA$ is still time-invariant, the resulting $\overline{\mA}_l$ is not fixed because the discretization of $\mA$ depends on the time interval $\Delta_l$. In addition, the matrix $\mA$ is defined to be diagonal for efficiency reasons.

It is assumed that the sequence $x_l$ is sampled from an underlying continuous function $x ( \cdot )$. Then, $\Delta_l$ can be interpreted as time intervals between points $t_l$ at which the sequence elements are sampled as
\begin{equation}
    x_l = x \left( t_l \right) = x \left( \sum\nolimits_{i = 0}^{l} \Delta_i \right). \label{eq:time_intervals_sampling}
\end{equation}

\subsection{Learned Selectivity as a Linear Approximation of Information}
\label{sec:learned_selectivity_approx}

\cite{gu2023mamba} introduce selectivity as a strategy to improve the expressivity of SSMs. We extend their work by elucidating how the time intervals $\Delta_l$ learnt by Mamba serve as linear approximators of information contained in sequence elements $x_l$. We define the information in this context as the change in the conditional probability density function $p(y | \{ x_l \} )$ of an outcome $y$ for an observed sequence $\{ x_l \} = [ x_1, \ldots, x_i, \ldots, x_l ]$ and $l \leq L$, which we quantify as a difference in likelihoods with $(x_i \in \{ x_l \})$ and without $(x_i \not\in \{ x_l \})$ the $i$-th element. The effect of the change in likelihoods can be observed by studying the distribution $p ( y | \{ x_l \}_{l \neq i} )$ obtained by removing $x_i$ from the observations.



In the context of SSMs, the target distribution is modeled as a parametric one $p^\ast ( y^\ast | \vh_L )$ where $\vh_L$, the last state vector, is obtained from (\ref{eq:discrete_ssm_ode}). If $\vh_L^{i}$ denotes the last state vector obtained from the sequence $\left\{ x_l \right\}_{l \neq i}$, the distribution $p^\ast ( y^\ast | \vh_L^{i})$ is conditioned on the sequence without $x_i$. Thus, the information in $x_i$ can be quantified by the change between $p^\ast ( y^\ast | \vh_L )$ and $p^\ast ( y^\ast | \vh_L^{i})$.

In the context of information theory \citep{cover2012elements}, the difference between two distributions is usually measured as the Kullback-Leibler (KL) divergence. Because $p^{*}$ is parametric on $\bm{h}_L$ in a form that depends on the task and architecture at hand, we can not analytically express the KL divergence between $p^\ast ( y^\ast | \vh_L )$ and $p^\ast ( y^\ast | \vh_L^{i})$. However, on the assumption that all transformations involved are continuous (as they often are), the difference is related to the Euclidean distance between parameters $\| \vh_L - \vh_L^{i} \|$. That justifies the validity of the result in Proposition~\ref{prop:linear_inf_approx}.


\begin{proposition} \label{prop:linear_inf_approx}
    For the discrete SSM defined by equations (\ref{eq:discrete_ssm_ode}-\ref{eq:discrete_ssm_output}) and (\ref{eq:zoh_disc}), when the matrix $\mA$ is diagonal, the distance $\| \vh_L - \vh_L^{i} \|$ is asymptotically linear with respect to $\Delta_i$ as $\Delta_i \to 0$ i.e.
    \begin{equation}
        \| \vh_L - \vh_L^{i} \| \sim c \Delta_i, \quad c=\mathrm{const}~\textrm{and}~\Delta_i \to 0.
    \end{equation}
\end{proposition}
\begin{proof}
    The proof is provided in Appendix \ref{app:app_selectivity}.
\end{proof}


We remark that the result does not depend on how $\Delta_i$ is obtained. Hence, we can tweak the function that computes the time intervals based on the behavior we are trying to model. In the next section,  we present our SeRpEnt and explain the benefits of computing $\Delta_i$ differently from Equation (\ref{eq:mamba_deltan}).


\section{Proposed Method}
\label{sec:method}

In this section, we describe SeRpEnt that compresses sequences before the application to a SSM. To do so, we leverage the two perspectives we develop for the time intervals $\Delta_l$: either as intervals between sampling points of the sequence (\ref{eq:time_intervals_sampling}) or as linear approximators of information in Section~\ref{sec:learned_selectivity_approx}.

\subsection{Sequence Compression}

Given a sequence $\left\{ x_l \right\}_{l \leq L}$, the length $L$ has a direct impact on the modeling capabilities of a SSM. A general input sequence is an element of $\ell^0 = \{ \{ x_l \}_{l \in \sN} \; | \; \exists L, \, \forall l > L, \, x_l = 0 \}$, the space of eventually-zero sequences. Because $\ell^0$ is infinite-dimensional, the outcome conditional distribution $p ( y | \{ x_l \}_{l \leq L} )$ is parameterized on an infinite-dimensional space. On the contrary, the predicted distribution $p^\ast ( y^\ast | \vh_L )$ of an SSM is parameterized on $\vh_L\in \sR^{N}$, a finite-dimensional representation of the whole sequence. Intuitively, the larger the length $L$, the less can be compressed into $\vh_L$, and, thus, the less expressive the SSM becomes.

The main motivation of this work is to compress an input sequence into a shorter one, while trying to incur a reduced loss of information. More formally, the problem can be described as follows.

\begin{problem}
    Given a sequence $\left\{ x_l \right\}_{l \leq L}$ and a length $\overline{L} < L$, find a sequence $\left\{ \overline{x}_l \right\}_{l \leq \overline{L}}$ such that the Kullback-Leibler (KL) divergence between $p ( y | \{ x_l \}_{l \leq L} )$ and $p ( y | \{\overline{x}_l \}_{l \leq \overline{L}} )$ is minimized.
\end{problem}

We propose a solution to the problem that combines heuristics with the result claimed in Proposition~\ref{prop:linear_inf_approx}. Then, we support the effectiveness of our approach empirically in Section \ref{sex:experiments}.

\subsection{Selective Resampling for Compression}
\label{sec:selective_resampling}


\begin{figure}
\includegraphics[width=0.95\linewidth]{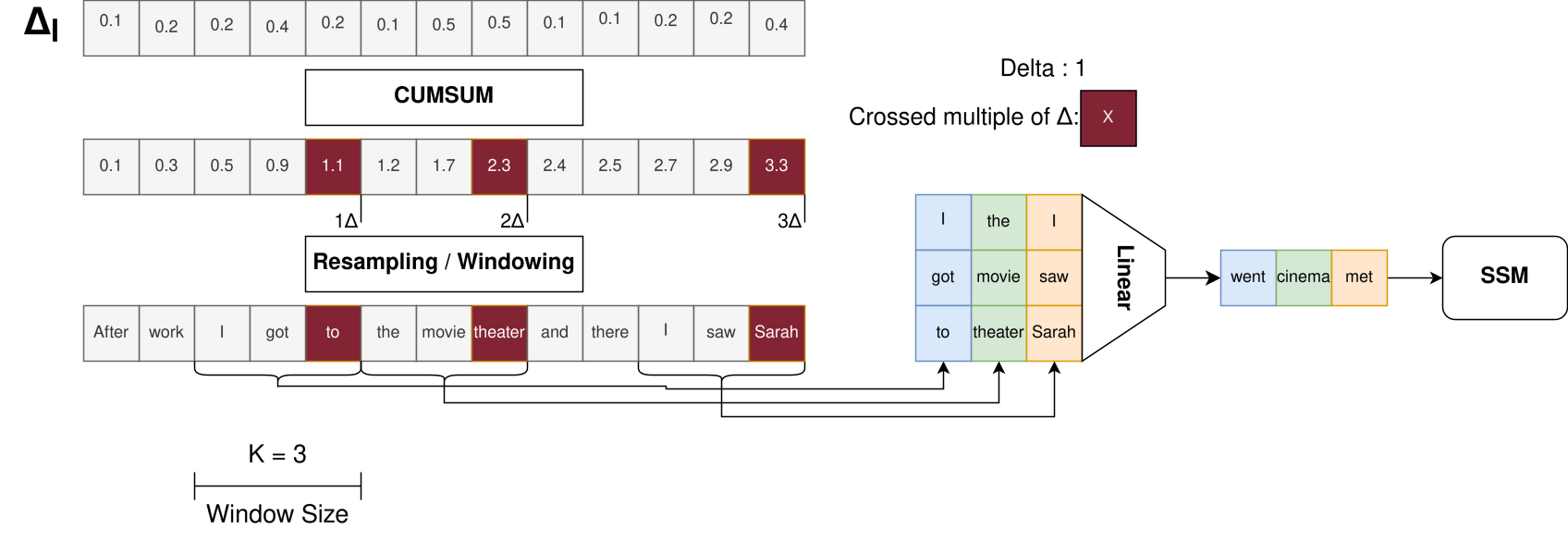}
\caption{(\textbf{Compression}) During the compression procedure through selective resampling a time interval $\Delta_l$ is computed for each element $x_l$ of the sequence. Then, the intervals are used to obtain the times at which elements sampled from the underlying function. Based on the $\Delta_l$ value and $\Delta$, elements are resampled and processed with a linear transformation together with their $K$ closest neighbors. Finally, the compressed sequence is the input to a SSM.}
\label{fig:serpent}
\end{figure}

As in Mamba, we start by computing a time interval value $\Delta_l$ for each $x_l$ in the sequence. From the discretization perspective, the time intervals are related to the times $t_l$ at which elements $x_l$ are sampled from the underlying continuous function $x\left( \cdot \right)$ in (\ref{eq:time_intervals_sampling}). In addition, Proposition~\ref{prop:linear_inf_approx} states that the $\Delta_l$ are related to the information contained in the elements $x_l$. This motivates our idea: to compose elements together based on their information content, \emph{i.e.} the value of $\Delta_l$.

In practice, we resample the sequence $\left\{ x_l \right\}_{l \leq L}$ into a compressed one $\left\{ \overline{x}_l \right\}_{l \leq \overline{L}}$ in which close elements are composed together. We achieve this by taking $\overline{x}_l$ to be uniformly sampled from the same underlying function:
\begin{equation}
    \overline{x}_l = x \left( \overline{t}_l \right) = x \left( l {\Delta} \right),
\end{equation}
where the $\overline{t}_l$ are the resampled times with a constant time interval ${\Delta} = \overline{t}_l - \overline{t}_{l - 1}$. Since we can not access the underlying function $x \left( \cdot \right)$, we infer $\overline{x}_l$ through interpolation.

For both $x_l$ and $\overline{x}_l$, we know the times at which they are sampled from $x \left( \cdot \right)$: $t_l = \sum\nolimits_{i = 0}^{l} \Delta_i$ and $\overline{t}_l = l \Delta$, respectively. This enables the adoption of a nearest neighbors procedure \citep{hastie2009elements} for interpolation. Given a window size $K$, for each $\overline{t}_l$, $\mathcal{N}_{K} \left( \overline{t}_l \right)$ is the set of the its $K$ closest neighbors among the input times $\left\{ t_l\right\}_{l \leq L}$. The interpolation procedure infers $\overline{x}_l$ as a function $\Gamma$ of its nearest neighbors.


Figure \ref{fig:serpent} illustrates the compression operation in our method. In particular, we select $\Gamma$ to be a linear transformation $\theta_{\Gamma}$ applied on the concatenation $\bigoplus$ of the inputs $x_k$ and their processed distances $\varepsilon \left( \overline{t}_l - t_k \right)$ from the resampled point can be expressed as
\begin{subequations}
\begin{align}
    \overline{x}_l & = \Gamma \left( \left\{ \left(x_k, t_k \right) \; | \; t_k \in \mathcal{N}_K \left( \overline{t}_l \right) \right\} \right) \label{eq:gamma_interpolation}, \\
    \overline{x}_l & = \theta_{\Gamma} \left( \bigoplus\nolimits_{t_k \in \mathcal{N}_K \left( \overline{t}_l \right)} \left[ x_k, \varepsilon \left( \overline{t}_l - t_k \right)\right] \right) ~\textrm{with }~ \varepsilon \left( d \right)_i = \exp(-\left( d - \mu_i \right)^2). \label{eq:linear_interpolation}
\end{align}
\end{subequations}

\paragraph{Gaussian basis expansion of time differences.}

We borrow a trick from the literature of graph neural networks \citep{NIPS2017_303ed4c6} to represent the time differences $d_{lk} = \overline{t}_l - t_k$. If the time difference is processed as a scalar, the random initialization of linear layers makes the information processed at different distances highly correlated. To avoid that, we expand the time difference $d_{lk}$ in a vector of Gaussian coefficients $\varepsilon ( d_{lk} )$ (\ref{eq:linear_interpolation}) where the means $\left\{ \mu_i \right\}_{i \leq G}$ are randomly initialized and learnable, while the dimension $G$ of the expansion is a hyperparameter of the model.

\paragraph{Resampling and compression rate.}

The most important component of the resampling procedure is the computation of the time intervals $\Delta_l$. The computation must ensure that the output sequence is a compression of the input, which is guaranteed if $\Delta_l \leq \Delta$. A hyperparameter $0 < \kappa < 1$ also controls the minimum compression rate, \emph{i.e.} $\overline{L} \geq \kappa L$, which can be ensured by posing $\Delta_l \geq \kappa \Delta$. We then employ the following expression for computing the time intervals
\begin{equation}
\label{eq:delta_serpent}
    \Delta_l = \sigma \left( \theta_{\Delta} \left( x_l \right) \right) \Delta \left( 1 - \kappa \right) + \kappa \Delta,
\end{equation}
where $\sigma$ is the sigmoid function, $\theta_{\Delta}$ is a linear transformation, and $\Delta$ is a learnable parameter. Because $0 < \sigma ( \theta_{\Delta} ( x_l ) ) < 1$, we have $\kappa \Delta \leq \Delta_l \leq \Delta$.



\subsection{SeRpEnt Architecture}
\label{sec:serpent_architecture}

We have described the inner functioning of the selective resampling procedure for compression of a sequence $\{ x_l \}_{l \leq L}$ into $\{ \overline{x}_l \}_{l \leq \overline{L}}$. We now present how that fits into the overall architecture of SeRpEnt. We remind that the choice of $\kappa$ in equation (\ref{eq:delta_serpent}) tunes the compression rate of the resampling procedure. Then, our SeRpEnt block compresses a sequence at different compression rates $\kappa_b$ and applies a separate SSM to each of them. Finally, each output sequence is decompressed and all of them are concatenated together. We further explain these steps in details with the overall architecture being depicted in Figure \ref{fig:serpent_arch}.

\begin{figure}
\includegraphics[width=0.95\linewidth]{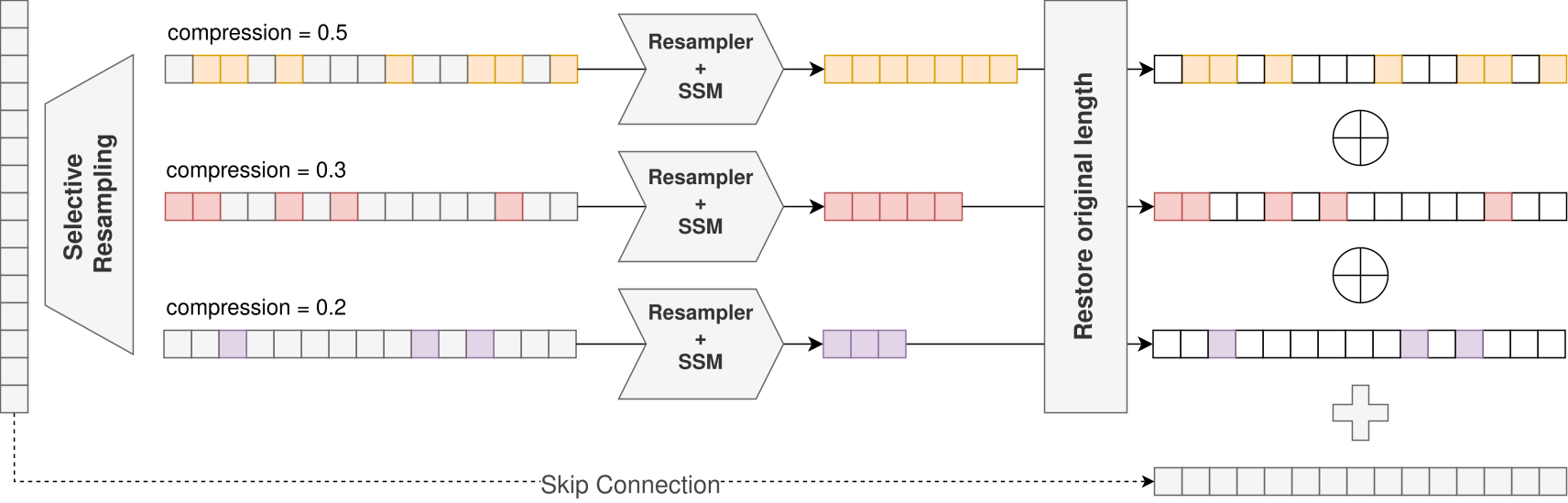}
\caption{(\textbf{Architecture}) Our SeRpEnt block works by compressing an input sequence $\{ x_l \}_{l \leq L}$ at different compression rates. It separately processes each of the compressed inputs using a SSM layer, decompresses the outputs, and concatenates them back into a single sequence with a skip connection.}
\label{fig:serpent_arch}
\end{figure}

\paragraph{Compression through selective resampling.} In each $b$-th SeRpEnt block, we perform this operation as covered in Section \ref{sec:selective_resampling}, in parallel for all blocks, with different compression rate $\kappa_b$. Hence, the input of this component is the sequence $\{x_l\}_{l \leq L}$ and its output is a series of compressed sequences $\{\overline{x}_l^b\}_{l \leq \overline{L}^{b}}$, one for every compression rate. This allows us to efficiently model short- and long-range dependencies by aggregating together elements based on their information content.

\paragraph{State space model.}
A SSM processes each compressed sequence $\{ \overline{x}_l^b \}_{l \leq \overline{L}^{b}}$ to obtain output sequence $\{ \overline{y}_l^{b} \}_{l \leq \overline{L}^{b}}$ with the the same length. We emphasize that SSM choice is independent from our resampling procedure. Therefore, our method can serve as an architectural extension for other SSM layer variants. We leverage this in the Section \ref{sec:experiments} experiments.

\paragraph{Reverse resampling.}
This step decompresses the sequences $\{ \overline{y}_l^{b} \}_{l \leq \overline{L}^{b}}$ to the original length $L$. Formally, for each individual $\{\overline{y}_l^{b} \}_{l \leq \overline{L}^{b}}$, the operation is equivalent to the compression one, but with the original and the resampled times being inverted. For each output time $t_l$, we find its nearest neighbors $\mathcal{N}_K (t_l)$ among $\{ \overline{t}_l \}_{l \leq \overline{L}}$ and express $y_l$ as a function $\Gamma^{*}$ of $\{ (\overline{y}_k, \overline{t}_k) \; | \; \overline{t}_k \in \mathcal{N}_K (t_l) \}$ from equation (\ref{eq:gamma_interpolation}). In practice, we find $\Gamma^{*}$ to work well by simply copying the closest $\overline{y}_k^b$ to $y_l$. The following equations recap all the operations in the SeRpEnt block:

\begin{minipage}{0.6\textwidth}
\begin{subequations}
\begin{align}
    \overline{x}^b_l & = \theta_{\Gamma} \left( \bigoplus\nolimits_{t_k \in \mathcal{N}_K \left( \overline{t}^i_l \right)} \left[ x_k, \varepsilon \left( \overline{t}^b_l - t_k \right)\right] \right), \\
    y^b_l & = \overline{y}^b_k \textrm{ with } k = \textrm{argmin}_{j} \left| t_l - \overline{t}_j^b \right|,
\end{align}
\end{subequations}
\end{minipage}
\hfill
\begin{minipage}{0.3\textwidth}
\begin{subequations}
\begin{align}
    \overline{y}^b & = \textrm{SSM} \left( \overline{x}^b \right), \\
    y_l & = \left( \bigoplus\nolimits_b y_l^b \right) + x_l.
\end{align}
\end{subequations}
\end{minipage}

Overall SeRpEnt network is a concatenation of individual blocks. We do not interleave blocks with linear layers due to the fact that features get linearly mixed during the resampling step in (\ref{eq:linear_interpolation}).

\section{Experiments}
\label{sex:experiments}

We empirically measure the performance of SeRpEnt on tasks of sequence classification and language modeling. In the first case, we evaluate SeRpEnt against a subset of the Long Range Arena (LRA) \citep{tay2020long} baselines, following previous works \citep{gu2022efficiently, gu2022parameterization, smith2023simplified}. In the second case, SeRpEnt is evaluated in language modeling using the experiment setup from \citet{fu2023hungry}.

All the details regarding experiments are reported in Appendix \ref{app:experimental_details}. We note here that we tested SeRpEnt against models with a comparable number of parameters. As explained in Section \ref{sec:serpent_architecture}, SeRpEnt is implemented independently from the SSM layer, which we apply depending on the task to: S4 \citep{gu2022efficiently} for LRA and Mamba \citep{gu2023mamba} for language modeling.

In our experiments, we observed that SeRpEnt improves the performance of the baseline it is applied on every task except for image ones in LRA. We believe the subset of image tasks in LRA (\textsc{Image}, \textsc{Pathfinder}, and \textsc{Pathfinder-X}) has a very different structural bias than the ones that sequence modeling problems manifest, and our experiments suggest SeRpEnt is incapable of capturing it. Exploring and justifying the shift in accuracy is cause for future research.

\subsection{Long Range Arena}

LRA is the benchmark composed of six sequence classification tasks which share one common property: modeling dependencies among distant sequence elements. Table \ref{tab:lra_table} shows results for SeRpEnt and other models the LRA tasks that involve sequence modeling, excluding those whose input are images. The reported score is accuracy. We are interested in comparing against an SSM architecture to its SeRpEnt variant. For reasons of completeness, we also report state-of-the-art results for models with quadratic complexity in the first three lines i.e. Mega \citep{ma2022mega}, ChordMixer \citep{khalitov2023chordmixer} and SeqBoat \citep{ren2023sparse}.

\setlength{\tabcolsep}{6pt}

\begin{table}[t]
    \centering
    \scalebox{0.88}{
        \begin{tabular}{m{3cm} m{1.4cm} m{1.4cm} m{1.7cm} | m{1.4cm}}
            \toprule
            {\centering \scalebox{1.2}{M}\small ODEL} & 
            {\centering \scalebox{1.2}{L}\small IST\scalebox{1.2}{O}\small PS} & 
            {\centering \scalebox{1.2}{T}\small EXT} & 
            {\centering \scalebox{1.2}{R}\small ETRIEVAL} & 
            {\centering \scalebox{1.2}{A}\small VG} \\ \midrule
            Mega & 63.14 & 90.43 & 91.25 & 81.01 \\ 
            ChordMixer & 59.89 & 88.87 & 90.38 & 79.71 \\ 
            SeqBoat & 61.70 & 89.60 & 91.28 & 80.86 \\ \midrule
            S4  & \textbf{59.60} & 86.82& 90.90 & 79.11  \\
            SeRpEnt+S4 & \textbf{59.60} & \textbf{88.60} & \textbf{91.18} & \textbf{79.79} \\
            \midrule
            S5  & 59.70 & 86.63 & \textbf{90.49} & 79.04  \\
            SeRpEnt+S5 & \textbf{60.35} & \textbf{89.62} &  89.04 & \textbf{79.67} \\
            \midrule
            Liquid-S4& \textbf{60.65} & 88.37 & 88.17 & 79.06 \\
            SeRpEnt+Liquid-S4 & 60.35 & \textbf{88.39} & \textbf{88.81} & \textbf{79.18} \\
            \bottomrule
        \end{tabular}
    }
    \caption{\textbf{Long Range Arena} Accuracy on individual LRA tasks  (excluding image ones) and the average, \%. Bold values denote an improvement from the baseline to SeRpEnt. On average, SeRpEnt shows a performance improvement over the base SSM models.}
    \label{tab:lra_table}
\end{table}

\subsection{Language Modeling}

For language modeling, at the time of writing, Mamba \citep{gu2023mamba} is the SSM variant with the best empirical performance. With the given computational resources at our disposal, we compare SeRpEnt+Mamba with a version of Mamba smaller than the ones evaluated in the original work and trained on WikiText-103-v1 \citep{merity2016pointer}, rather than the Pile \citep{pile}.

We train S4, Mamba, and our SeRpEnt on WikiText-103-v1 and report their performance metrics in Table \ref{tab:lm_table}. We observe that SeRpEnt outperforms Mamba despite the comparable number of parameters. This supports the efficacy of the proposed architecture. All models are trained on 4 Tesla V100 32Gb GPUs for a time ranging from two to four days.

With the size of the models and datasets in Table \ref{tab:lm_table}, we do not compare to large language model (LLM) baselines \citep{Zhao2023ASO}. However, these results support our claim that SeRpEnt introduces a positive impact on the Mamba architecture. To verify this, we meticulously match their hyperparameters to only measure the added performance as described in Appendix \ref{app:experimental_details}.

\begin{table}[h]
    \centering
    \scalebox{0.88}{
        \begin{tabular}{m{3cm} m{2cm} | m{2cm} m{2cm} m{2cm} m{2cm} m{2cm}}
            \toprule
            {\centering \scalebox{1.2}{M}\small ODEL} & 
            {\centering \scalebox{1.2}{M}\small ODEL \scalebox{1.2}{S}\small IZE} &
            {\centering \scalebox{1.2}{T}\small OP-1 ACC.$\uparrow$} & 
            {\centering \scalebox{1.2}{T}\small OP-5 ACC.$\uparrow$} & 
            {\centering \scalebox{1.2}{L}\small OSS$\downarrow$} & 
            {\centering \scalebox{1.2}{P}\small ERPLEXITY$\downarrow$} \\ \midrule
            S4 &  35.2M & 35.5 &  56.2 &  3.72 & 557 \\ 
            Mamba & 44.3M & 37.1 & 58.3 & 3.54 & \textbf{46} \\
            SeRpEnt+Mamba& 47.4M & \textbf{38.3} & \textbf{58.7} & \textbf{3.51} & \textbf{46}\\
            \bottomrule
        \end{tabular}
    }
    \caption{\textbf{Language Modeling on WikiText-103v1} The scores are for the dataset's test split. We compute performance metrics based on the official implementation of S4 and Mamba baselines. The proposed SeRpEnt's resampling mechanism with sequence compression improves the performance metrics (1.2\% and 0.4\% top-1 and top-5 accuracy, respectively) when added to the base model.}
    \label{tab:lm_table}
\end{table}


\label{sec:experiments}

\section{Conclusion}
\label{sec:conclusions}

Inspired by our analysis of the selectivity mechanism, we introduced SeRpEnt, a method for compressing and decompressing sequences in SSMs. We showed how the time intervals learned by Mamba are linear approximators of information. This result justified the behavior of the proposed SeRpEnt that extends Mamba by compressing sequences in an information-aware fashion. We also showed SeRpEnt advantages empirically using the long-range arena benchmark and in other language modeling tasks. The proposed approach can benefit both current and future models when used as an architectural component that is orthogonal to the recent developments in other SSM variants.

\section{Acknowledgments}

The authors acknowledge financial support from the Panasonic Corporation and wish to thank the scientists at Panasonic for insightful discussions and recommendations.

\section{Data Availability Statement}

The data used to obtain the results of this paper can be obtained from the Long Range Arena repository \cite{tay2020long} and the Wikitext dataset \cite{merity2016pointer} accessible through the Hugging Face Hub.


\newpage
\begin{appendices}

\section{Learned Selectivity as a Linear Approximation of Information}
\label{app:app_selectivity}

In this section we prove Proposition~\ref{prop:linear_inf_approx} from Section \ref{sec:learned_selectivity_approx}. For ease of readability, we use the subscript $m$ instead of $i$ to express the individual element we are computing the information content of.

\begin{proof}
By using equation (\ref{eq:zoh_disc}) and taking into account that the matrix $\mA$ is fixed, we start with the transformed sequence $\{\vx_l = \mB_i x_l \}_{l \leq L}$ which reduces the recurrence (\ref{eq:discrete_ssm_ode}) to
\begin{equation}
\vh_l = \exp(\Delta_l \mA) \vh_{l - 1} + \left( \Delta_l \mA \right)^{-1} \left( \exp(\Delta_l \mA) - \mI \right) \Delta_l x_l.
\end{equation}

Because the matrix $\bm{A}$ is diagonal, for every component $1 \leq j \leq N$, the recurrence takes the form
\begin{equation}
\label{eq:proof_component_j}
    \left( \vh_l \right)_j = \exp(\Delta_l \mA_{jj} ) \left( \vh_{l - 1} \right)_{j} + \frac{\exp(\Delta_l \mA_{jj}) - 1}{\Delta_l \mA_{jj}} \Delta_l \left( \vx_l \right)_j.
\end{equation}
In particular, the recurrence (\ref{eq:proof_component_j}) is separable across components. Next, we study the recurrence for a single component $j$, but we drop the subscript $j$ for readability and also pose $\alpha = \mA_{jj}$. We get back to the vector form at the end of the proof. Thus, by canceling out the $\Delta_l$, equation (\ref{eq:proof_component_j}) can be written as
\begin{equation}
\label{eq:proof_component_k}
    \vh_l = \exp(\alpha \Delta_l) \vh_{l - 1} + (\exp(\alpha \Delta_l) - 1) \bm{x}_l / \alpha,
\end{equation}
where we introduce the $T_l = (\exp(\alpha \Delta_l) - 1) \bm{x}_l / \alpha.$

Then, the recurrence (\ref{eq:proof_component_k}) can be further expressed by
\begin{align}
\vh_l & = \exp \left(\alpha \Delta_l \right) \vh_{l - 1} + T_n \\
\vh_{n + 1} & = \exp \left(\alpha \left( \Delta_l + \Delta_{l + 1} \right)\right) \vh_{l - 1} + \exp \left(\alpha \Delta_{l + 1} \right) T_l + T_{l + 1} \\
\vdots & \notag \\
\vh_{l + k} & = \exp \left( \alpha \sum\nolimits_{i = 0}^{k} \Delta_{l + i} \right) \vh_{l - 1} + \sum\nolimits_{i = 0}^{k} \exp \left( \alpha \sum\nolimits_{j = i + 1}^{k} \Delta_{l + j} \right) T_{l + i}.
\end{align}

The last state space vector is
\begin{equation}
    \vh_L = \exp \left( \alpha \sum\nolimits_{i = 0}^{L - m + 1} \Delta_{m + i} \right) \vh_{m - 1} + \sum\nolimits_{i = 0}^{L - m + 1} \exp \left(\alpha \sum\nolimits_{j = i + 1}^{L - m + 1} \Delta_{m + j} \right) T_{m + i}.
\end{equation}

Similarly
\begin{equation}
    \vh_L^{m} = \exp \left( \alpha \sum\nolimits_{i = 1}^{L - m + 1} \Delta_{m + i} \right) \vh_{m - 1} + \sum\nolimits_{i = 1}^{L - m + 1} \exp \left(\alpha \sum\nolimits_{j = i + 1}^{L - m + 1} \Delta_{m + j} \right) T_{m + i}.
\end{equation}

We now express the difference $\vh_L - \vh_L^{m}$ in which we see that all terms in the sum except that for $i = 0$ cancel out
\begin{align}
    \vh_L - \vh_L^{m} 
    & = \exp \left( \alpha \sum\nolimits_{i = 1}^{L - m + 1} \Delta_{m + i} \right) \left( \exp \left( \alpha \Delta_m \right) - 1\right) \left( \vh_{m - 1} + \vx_m / \alpha \right).
\end{align}

And we can see that, since it depends on $\Delta_m$ only through the factor $\exp (\alpha \Delta_m) - 1$, the difference is asymptotically linear with respect to $\Delta_m$ for $\Delta_m$ approaching zero
\begin{align}
    \vh_L - \vh_L^{m} & \sim c \Delta_m  \quad \left( \text{as } \Delta_m \to 0 \right) \\
    c & = \exp \left( \alpha \sum\nolimits_{i = 1}^{L - m + 1} \Delta_{m + i} \right) \left( \vh_{m - 1} + \vx_m / \alpha \right).
\end{align}

We now go back to the vector case and we see that, for an individual component $j$ we have
\begin{align}
    \left( \vh_L - \vh_L^{m} \right)_j & \sim c_j \Delta_m \quad \left( \text{as } \Delta_m \to 0 \right) \\
    \left( \vh_L - \vh_L^{m} \right)_j & = c_j \Delta_m + o \left( \Delta_m \right).
\end{align}

It follows
\begin{equation}
    \| \vh_L - \vh_L^{m} \| = \sqrt{ \left( c_1 \Delta_m + o \left( \Delta_m \right) \right)^{2} + \cdots \left( c_N \Delta_m + o \left( \Delta_m \right) \right)^2}
\end{equation}
and as $\Delta_m \to 0$ the terms $o \left( \Delta_m \right)$ become negligible, hence
\begin{equation}
    \| \vh_L - \vh_L^{m} \|^2 = C \Delta_m^{2} \text{ where } C = \left( \sum\nolimits_{j = 1}^{N} c_j \right)^{2}
\end{equation}
from which it follows
\begin{align}
    \| \vh_L - \vh_L^{m} \| & \sim c \Delta_m \quad ( \text{as } \Delta_m \to 0) \\
    c & = \sum\nolimits_{j = 1}^{N} c_j
\end{align}
which concludes the proof.
\end{proof}

\section{Experimental Details}
\label{app:experimental_details}

We provide additional details for the experiments presented in Section \ref{sec:experiments}. We ran all experiments using the AdamW optimizer. Because SeRpEnt is built on top of either S4 or Mamba, we refer to \citep{gu2022efficiently, gu2023mamba} regarding the details about the presence of activation functions, gates, and dropout in the SSM layers.

\subsection{Long Range Arena}
We list the LRA hyperparameters for SeRpEnt in Table \ref{tab:hyperparams}. More details can be found in the configuration files in \verb|conf/experiment|. For every task the models are trained on a single A100GPU 64GB.


As described in Section \ref{sec:serpent_architecture} we use multiple parallel branches for resampling at each layer.
Among these SeRpent employs a branch without resampling, as a "\textit{base}" branch which runs a simple LTI SSM along the whole sequence.
We also find beneficial to use a Batch Normalization layer \citep{Ioffe:2015:BNA:3045118.3045167} after the skip connection.

\begin{table}[h]
    \centering
    \scalebox{0.88}{
        \begin{tabular}{m{2cm}| m{1.4cm} m{1.4cm} m{1.7cm} m{1.4cm} m{1.7cm} m{1.4cm}}
            \toprule
            {} &
            {\centering \scalebox{1.2}{L}\small IST\scalebox{1.2}{O}\small PS} & 
            {\centering \scalebox{1.2}{T}\small EXT} & 
            {\centering \scalebox{1.2}{R}\small ETRIEVAL}
            \\ \midrule
            \# Branches & 2 & 3 & 2 \\
            Compression & $[0.5]$ & $[0.5, 0.2]$ & $[0.5]$\\
            Window Size & 6 & 5 & 7\\
            Features H & 192 & 192 & 128\\
            $d_{inner}$ & 96 & 64 & 64 \\
            SSM $d_{state}$ & 4 & 4 & 4 \\
            LR & 0.001 & 0.001 & 0.007 \\
            Batch Size & 16 & 16 & 64 \\
            WD & 0.05 & 0.05 & 0.05 \\
            Epochs & 50 & 32& 20 \\
            Patience & 15 & 15 & 10 \\
            Scheduler & Cosine & Plateau & Cosine \\
            \midrule
            \textbf{Model Size} & 813K & 872K & 494K \\
            \bottomrule
        \end{tabular}
    }
    \caption{Hyperparameters for Long Range Arena}
    \label{tab:hyperparams}
\end{table}

\subsection{Language Modeling}

    We experiment with the baselines in language modeling on the WikiText-103-v1 \citep{merity2016pointer} dataset with the train, validation, and test splits provided by the repository on Hugging Face. The results in Table \ref{tab:lm_table} show evaluations performed on the test split.
    We train the byte-pair encoding (BPE) \citep{Gage1994ANA} tokenizer on the train split of the dataset with a maximum vocabulary length of $30,000$ and use it in all experiments and models.
    
    Following the same approach in previous works, we use RMSNorm layers \citep{zhang-sennrich-neurips19} for normalization, which we apply before every SeRpEnt block. For all three models, we employ a depth for the networks of $8$, a plateau scheduler with a patience of $5$ and a factor of $0.1$. We train for a total number of $50$ epochs and a batch size of $64$ without employing early stopping.
    
    In the case of Mamba, we employ an inner dimension of $512$. For SeRpEnt, the inner dimension is $510$ to split it into three parallel components in each individual SeRpEnt block. We use a base SeRpEnt block without compression and two other blocks with compression rates of $0.5$ and $0.1$, respectively. The dimension of the Gaussian basis expansion of the differences in Section \ref{sec:selective_resampling} is $8$. The learning rate for the three experiments is $0.0005$ and the weight decay is $0.05$.

\section{Training Curves}

\begin{figure}[H]
    \centering
    \subfloat[\textsc{ListOps}]{\includegraphics[width=0.3\textwidth]{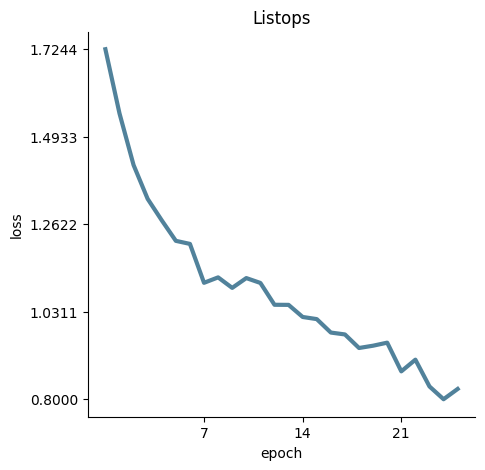}}\hfill
    \subfloat[\textsc{Text}]{\includegraphics[width=0.3\textwidth]{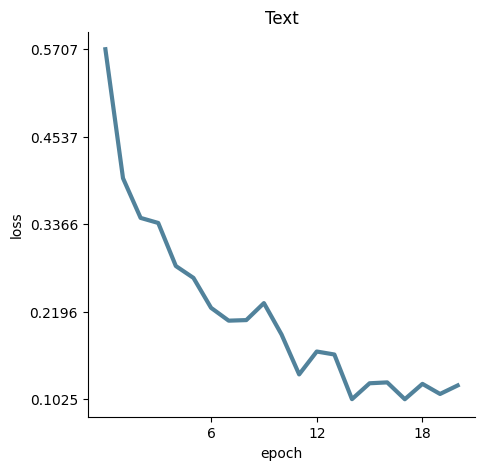}}\hfill
    \subfloat[\textsc{Retrieval}]{\includegraphics[width=0.3\textwidth]{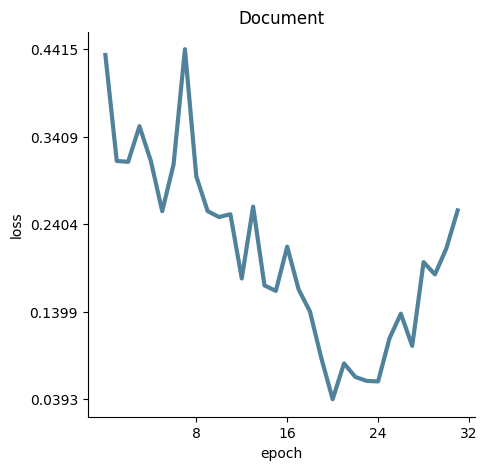}}
    
    
    \caption{Training losses during training for different tasks}
    \label{fig:train_loss}
\end{figure}

\end{appendices}

\bibliography{paper}

\end{document}